%% file: root.tex
\documentclass[letterpaper, 10 pt, conference,final]{ieeeconf}
\usepackage{siunitx}
\usepackage{subcaption}
\IEEEoverridecommandlockouts
\usepackage{url}
\usepackage{xspace}
\usepackage{amsmath}
\usepackage{siunitx}
\usepackage{amssymb}
\usepackage{hyperref}
\usepackage{float}
\usepackage{balance}

\usepackage{listings}
\usepackage{algorithm}
\usepackage{algpseudocode}

\usepackage[export]{adjustbox}

\usepackage{graphicx}
\usepackage{caption}
\usepackage{tikz}
\usetikzlibrary{positioning}
\usetikzlibrary{arrows.meta}
\usepackage{makecell}
\usepackage{tabulary}
\usepackage{rotating}
\usepackage{booktabs}
\usepackage{multirow}
\usepackage{wrapfig}
\newtheorem{theorem}{Theorem}

\input{src/definitions}

\usepackage[export]{adjustbox}

\begin{document}

\title{\Huge Sparse Multilevel Roadmaps \\for High-Dimensional Robot Motion Planning}

\input{images/pullfigure/pullfigure_abstract}

\input{src/abstract}
\input{src/introduction}
\input{src/related_work}
\input{src/background}

\input{src/method}

\input{src/evaluations}
\input{src/conclusion}

\bibliographystyle{IEEEtranS}
{
\balance
\small
\bibliography{IEEEabrv, bib/general}
}
\end{document}

%% file: src/definitions.tex
\newcommand{\algorithmName}{SMLR\xspace}
\newcommand{\X}{X}
\newcommand{\Xf}{\X_{\text{free}}}

\newcommand\R{\mathbb{R}}

\def\Xtop{\X_{\text{top}}}
\def\Xcur{\X_{\text{cur}}}
\newcommand{\Xk}{\X_{k}}
\newcommand{\Xkk}{\X_{k-1}}
\newcommand{\G}{G}
\def\Gtop{\G_{\text{top}}}

\def\PriorityQueue{\ensuremath{\mathbf{X}}}

\def\x{x}

\def\xi{\x_I}
\def\xg{\x_G}
\def\xr{\x_{\text{rand}}}

\def\fiber{F}
\def\xf{\x_{\fiber}}
\def\xb{\x_{B}}

\def\visRegion{\ensuremath{\delta}}
\def\bias{\visRegion_{\text{bias}}}

%% file: images/pullfigure/pullfigure_abstract.tex
\author{Andreas Orthey$^{1}$
and Marc Toussaint$^{1,2}$%
}

\twocolumn[{%
\begin{@twocolumnfalse}
\centering
\maketitle
\includegraphics[width=0.32\textwidth,valign=m]{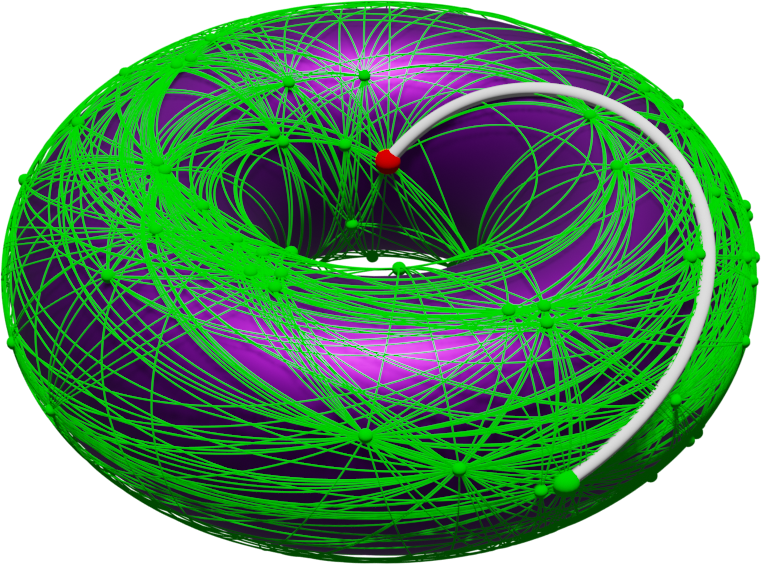}
\includegraphics[width=0.32\textwidth,valign=m]{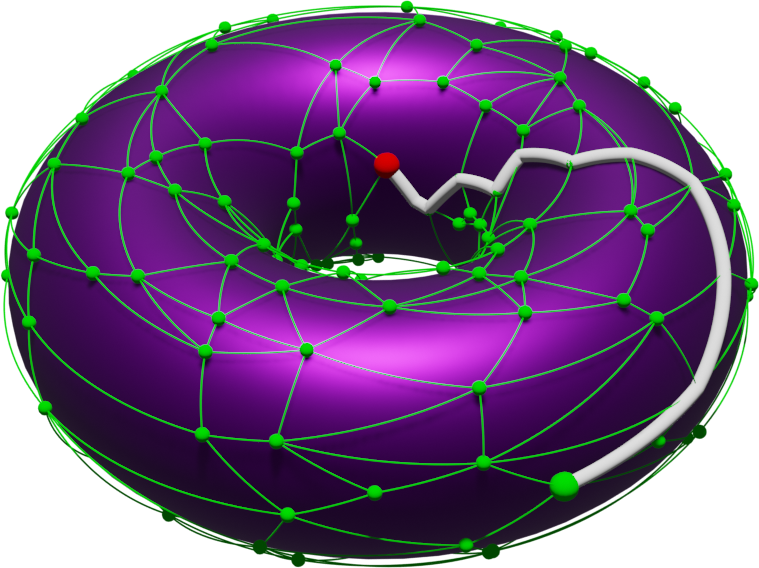}
\includegraphics[width=0.32\textwidth,valign=m]{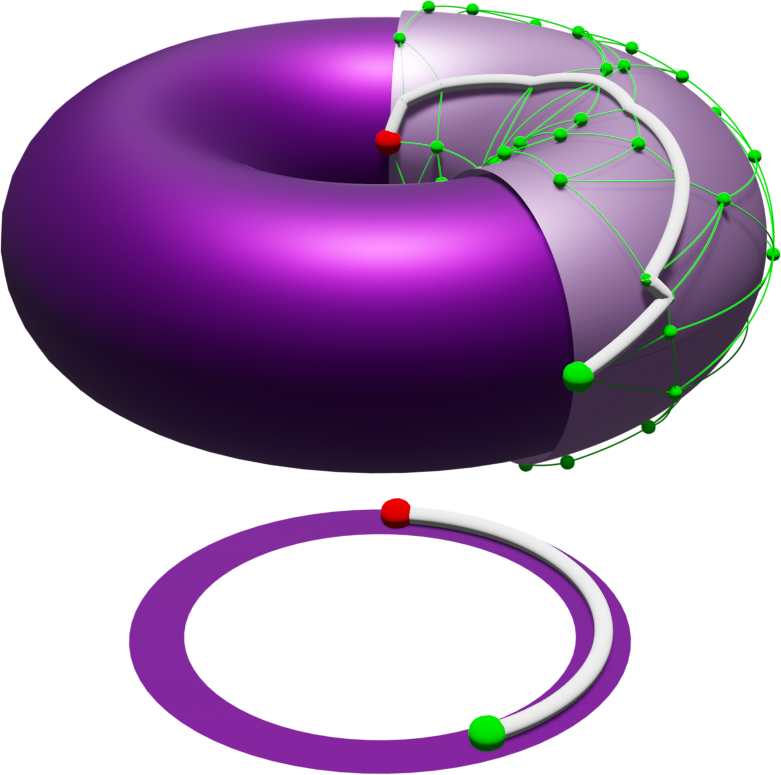}
\captionof{figure}{We generalize sparse roadmaps to fiber bundles. Here, we demonstrate this idea on the Torus $T^2 = S^1 \times S^1$ with $S^1$ being the circle. {\textbf{Left}}: Dense roadmap using probabilistic roadmap planner \cite{Karaman2011}. {\textbf{Middle}}: Sparse roadmap using sparse roadmap spanner \cite{dobson_2014}. {\textbf{Right}}: Sparse multilevel roadmap on fiber bundle $T^2 \rightarrow S^1$ using our algorithm (SMLR), which restricts sampling based on information from the lower-dimensional space $S^1$.\label{fig:pullfigure}}
\end{@twocolumnfalse}
}]

{
  \footnotetext[1]{Max Planck Institute for Intelligent Systems, Stuttgart, Germany. Marc Toussaint thanks the MPI-IS for the Max Planck Fellowship.}%
  \footnotetext[2]{Technical University of Berlin, Berlin, Germany  {\tt\footnotesize \{aorthey\}@is.mpg.de}, {\tt\footnotesize \{toussaint\}@tu-berlin.de}}
}

%% file: src/abstract.tex
\begin{abstract}

Sparse roadmaps are important to compactly represent state spaces, to determine problems to be infeasible and to terminate in finite time. However, sparse roadmaps do not scale well to high-dimensional planning problems. In prior work, we showed improved planning performance on high-dimensional planning problems by using multilevel abstractions to simplify state spaces. In this work, we generalize sparse roadmaps to multilevel abstractions by developing a novel algorithm, the sparse multilevel roadmap planner (SMLR). To this end, we represent multilevel abstractions using the language of fiber bundles, and generalize sparse roadmap planners by using the concept of restriction sampling with visibility regions. We argue SMLR to be probabilistically complete and asymptotically near-optimal by inheritance from sparse roadmap planners. In evaluations, we outperform sparse roadmap planners on challenging planning problems, in particular problems which are high-dimensional, contain narrow passages or are infeasible. We thereby demonstrate sparse multilevel roadmaps as an efficient tool for feasible and infeasible high-dimensional planning problems.

\end{abstract}

%% file: src/introduction.tex
\section{Introduction}

Sparse roadmaps \cite{dobson_2014} are essential in motion planning tasks to reduce model complexity and terminate motion planning in finite time, thereby providing (probabilistic) infeasibility proofs. Such infeasibility proofs are essential if we like to use a motion planner as building block for larger action skeletons \cite{Kaelbling2011} or symbolic planning systems \cite{Toussaint2018}. However, sparse roadmaps often operate on the full state space of the robot(s), thereby taking too much time to converge---making them often inapplicable for higher-dimensional systems.

To address this problem, we propose to use sparse roadmaps \cite{dobson_2014} in conjunction with multilevel abstractions of the state space \cite{Orthey2020IJRR}. By exploiting multilevel abstractions---which we model using fiber bundles \cite{steenrod_1951}---we can often terminate the algorithm significantly faster than state-of-the-art sparse roadmap planners operating on the full state space. 

While multi-resolution roadmaps exists \cite{Ichnowski2019, Saund2020}, we are not aware of any algorithm to compute sparse roadmaps over multilevel abstractions. We therefore believe to be the first to combine both concepts into one concise algorithm. Let us summarize our contributions as follows. 
\begin{enumerate}
    \item We present the Sparse MultiLevel Roadmap planner (\algorithmName), which generalizes sparse roadmaps \cite{dobson_2014} to efficiently exploit fiber bundle structures \cite{Orthey2020IJRR}
    \item We evaluate \algorithmName on eight challenging feasible and infeasible motion planning problems involving high-dimensional state spaces up to $34$-degrees of freedom (dof)
\end{enumerate}

%% file: src/related_work.tex
\section{Related Work}

We review two aspects of (sampling-based) motion planning \cite{lavalle_2006}. First, we discuss multilevel motion planning, where we plan over multiple levels of abstraction. Second, we discuss sparse roadmaps on general state spaces. We will investigate both topics in detail in Sec. \ref{sec:background}.

\subsection{Multilevel Motion Planning}

To efficiently solve high-dimensional motion planning problems, we can use the framework of multilevel motion planning 
\cite{Ferbach1997, Sekhavat1998, Reid2020, Vidal2019, Orthey2020IJRR}, where (admissible) lower-dimensional projections are used to simplify the state space of a robot. We can construct multilevel abstractions either manually \cite{Reid2019, Orthey2019} or learn them from data \cite{Ichter2019, Brandao2020}. Our approach is complementary, in that we assume a multilevel abstraction to be given and we concentrate on computing sparse roadmaps over those abstractions.

Once we fix a multilevel abstraction, we can utilize classical motion planning algorithms to exploit them. A popular choice is the rapidly-exploring random tree algorithm \cite{Kuffner2000}, which we can generalize to selectively grow samples towards regions informed by lower-dimensional abstractions \cite{Ichter2019, Orthey2019} or workspace information \cite{Rickert2014}. While algorithms often show speed-ups of two to three orders of magnitude \cite{Rickert2014, Tonneau2018}, they usually lack guarantees on asymptotic optimality \cite{Karaman2011}. There are, however, two planner which provide those guarantees. First, the quotient-space roadmap planner (QMP*)
\cite{Orthey2020IJRR, Orthey2018}, which generalizes the probabilistic roadmap planner (PRM*) \cite{Karaman2011}. Second, the hierarchical bi-directional fast marching tree (HBFMT*) \cite{Reid2019, Reid2020}, which generalizes the fast marching trees algorithm (FMT*) \cite{Janson2015}. While both guarantee asymptotic optimality \cite{Orthey2020IJRR, Reid2020}, they support, however, either only euclidean spaces \cite{Reid2020} or rely on dense roadmaps \cite{Orthey2020IJRR}. Our approach differs significantly, in that we are the first to compute sparse roadmaps over general multilevel abstractions---while providing guarantees on asymptotic near-optimality.

\subsection{Sparse Roadmaps}

The history of sparse roadmaps essentially begins with the pioneering work by Sim{\'e}on et al. \cite{Simeon2000}, who were the first to prune states based on visibility regions. With visibility regions, we try to find a minimal set of states from which the full state space is visible, similar to the concept of guards in the art gallery problem \cite{Orourke1987}. However, visibility roadmaps often sacrifice on path quality. As remedies, we could introduce cycles \cite{Schmitzberger2002, Nieuwenhuisen2004} or use edge visibility \cite{jaillet_2008} to improve path quality. 

While cycles and edge visibility can improve path quality, there are no guarantees on optimality. This changed with the advent of near-optimal sparse roadmaps \cite{Marble2013}. Using dense asymptotic optimal roadmaps \cite{Karaman2011}, we can use graph spanners to sparsify a dense roadmaps while providing guarantees on path quality. We can achieve this by either removing edges \cite{Marble2013, Wang2015} or edges and vertices \cite{Salzman2014}. Computing dense roadmaps before sparsification is, however, computationally expensive. Later work introduces incremental sparse graph spanners, with which we can remove dependence on dense roadmaps altogether \cite{dobson_2014}. Our work is complementary to sparse graph spanners, in that we also use incremental sparse graph spanners \cite{dobson_2014}. We differ, however, in building not one, but multiple sparse roadmaps on different abstraction levels. 

When using sparse roadmaps, we often face the problem of explicitly defining a visibility or connection radius to define the sparseness of the graph.
To handle this trade-off between optimality and efficiency, we can often create multi-resolution roadmaps \cite{Du2020}. Multi-resolution roadmaps are sets of roadmaps which differ in how sparse they are. To vary roadmap sparsity, we could change the connection radius \cite{Saund2020} or we can selectively remove edges, either evenly distributed \cite{Ichnowski2019} or based on a reliability criterion \cite{Murray2020}. To exploit those multi-resolution roadmaps, we could plan on the highest resolution roadmap and selectively refine the roadmap whenever we hit an obstacle \cite{Saund2020}. Such a strategy is efficient, because solutions on sparser roadmaps act as admissible heuristics for planning \cite{Aine2016, Du2020}. While multi-resolution roadmaps exist on the same state space, our approach is complementary, in that we create sparse multilevel roadmaps on different state spaces, whereby each state space represents a relaxed planning problem.

%% file: src/background.tex
\section{Background\label{sec:background}}

\input{images/restrictions/restrictions}

We develop an algorithm which grows sparse roadmaps over fiber bundles to efficiently exploit high-dimensional planning problems. As background for this task, we review the topics of optimal motion planning, multilevel abstractions (modelled using fiber bundles) and sparse roadmaps.

\subsection{Optimal Motion Planning}

Let $\X$ be an $n$-dimensional state space and let $\xi$ and $\xg$ be two states in $\X$ which we call the initial and the goal state. To each state space, we associate a metric function $d: \X \times \X \rightarrow \R$ and a constraint function $\phi: \X \rightarrow \{0,1\}$ which evaluates to zero if a state is feasible and to one otherwise. The state space thus splits into two components, the constraint-free subspace $\Xf = \{x \in \X \mid \phi(x) = 0\}$ and its complement. We define the optimal motion planning problem as the tuple $A = (\Xf, \xi, \xg, J)$, which requires us to design an algorithm to find a continuous path from $\xi$ to $\xg$ while (1) staying exclusively inside $\Xf$ and (2) minimizing the cost functional $J$ which maps paths in $\Xf$ to real numbers. 

We define a motion planning algorithm (a planner) as a mapping from $A$ to a path through $\Xf$. A planner can have different desirable properties. First, we like a planner to be \emph{probabilistically complete}, meaning the probability of finding a solution path if one exists approaches one as time goes to infinity. Second, we like a planner to be \emph{asymptotically near-optimal}, meaning the probability of finding a path is at least $\epsilon$ worse than the optimal solution path (under cost functional $J$). Third, we like a planner to be \emph{asymptotically sparse}, meaning the probability of adding new nodes and edges converges to zero if time goes to infinity \cite{dobson_2014}. 

\subsection{Multilevel Motion Planning}

Because state spaces are often too high-dimensional to plan in, we use multilevel abstractions which we model using fiber bundles \cite{steenrod_1951, lee_2003}. A fiber bundle is a tuple $(\X, B, F, \pi)$, consisting of a bundle space $\X$, a base space $B$, a fiber space $F$ and a projection mapping $\pi$ from $\X$ to B. We assume that both state space and base space have associated constraint functions $\phi$ and $\phi_B$ and that the projection mapping $\pi$ is admissible w.r.t. the constraint functions, i.e. $\phi_B(\pi(x)) \leq \phi(x)$ for any $x$ in $\X$ \cite{Orthey2019}. The admissibility condition ensures that we preserve feasible solution paths under projection. While we exclusively use product spaces in this work, we model them using fiber bundles since they provide a useful vocabulary (restrictions and sections) and since they are required for extensions to task-space projections.

Our approach uses the following three concepts. First, we define fibers over a base element $b$ in $B$ as $F(b) = \{x \in \X\mid \pi(x) = b\}$, which is the set of points in $\X$ projecting onto $b$. Please see Fig. \ref{fig:restriction:fiber} for an example of a fiber on the torus $T^2 = S^1 \times S^1$ with base space $S^1$. We additionally define the method  $\textsc{Lift}: B \times F \rightarrow \X$, which takes a base element $b$ and a fiber element $f$ in $F(b)$ to the bundle space. In the case of product spaces, we can define $\textsc{Lift}(b,f) = (b,f)$. Second, we define path restrictions over a base path $p: I \rightarrow B$ as $r(p) = \{x \in \X \mid \pi(x) \in p[I]\}$, whereby $I$ is the unit interval and $p[I]$ is the image of the base path in $B$. Please see Fig. \ref{fig:restriction:path}. Third, we define graph restrictions over a graph $G_B = (V_B, E_B)$ on $B$ as $r(G_B) = \{x \in \X \mid \pi(x) \in e[I], e \in E_B\}$ whereby $V_B$ are vertices in $B$, $E_B$ is the set of edges in $B$ and $e[I]$ is the image of an edge on the base space. Fig. \ref{fig:restriction:graph} provides a visualization of a graph restriction (individual edge restrictions have different distances from torus for better visualization). For more details, please see \cite{Orthey2020IJRR} or \cite{steenrod_1951}. 

\subsection{Sparse Roadmaps}

To grow a sparse roadmap, we use the algorithm by Dobson and Bekris \cite{dobson_2014}. The sparse roadmap planner is similar to probabilistic roadmaps \cite{Kavraki1996, Karaman2011}, but uses a visibility region $\visRegion$, which consists of all feasible states in the hypersphere of radius $\visRegion$ around a state, to prune samples. To implement the pruning step, we add a new feasible sample if and only if it fulfills a sparseness condition. 

The sparseness condition consists of four elementary tests \cite{dobson_2014}. First, we test for coverage, meaning we add the sample if it does not lie in the visibility region of any sample in the graph. Second, we test for connectivity, meaning we add the sample, if it lies in multiple visibility regions, which belong to disconnected components of the sparse graph. Third, we test for interfaces, meaning we add the sample, if it lies in multiple visibility regions, which are not yet connected by an edge. Fourth and finally, we test for shortcuts, meaning we add the sample, if it provides proof of a shorter path through the free state space. We terminate the algorithm, if we either find a feasible path or if we fail $M$ consecutive times to add a sample to the sparse roadmap. For more details please see \cite{dobson_2014}.

The sparse roadmap planner is probabilistically complete and asymptotically near-optimality \cite{dobson_2014} and depends on the following parameters. First, the visibility region $\visRegion$, which is usually a fraction of the measure of the state space. Second, the maximum number of consecutive failures $M$. $M$ is important in the analysis of the algorithm, because it provides a probabilistic estimation of the free state space covered, which is defined as the percentage $1-\frac{1}{M}$ \cite{simeon_2002}. As an example, if we stop with $M=100$, our probabilistic estimate of the free state space covered is $99\%$.  Finally, we have an additional parameter for testing for shortcuts, which provides a trade-off between optimality and efficiency \cite{dobson_2014}.

%% file: images/restrictions/restrictions.tex
\begin{figure*}[hbt!]
    \centering
    \begin{subfigure}[t]{0.33\textwidth}
    \centering
        \vskip 0pt
        \includegraphics[width=\textwidth]{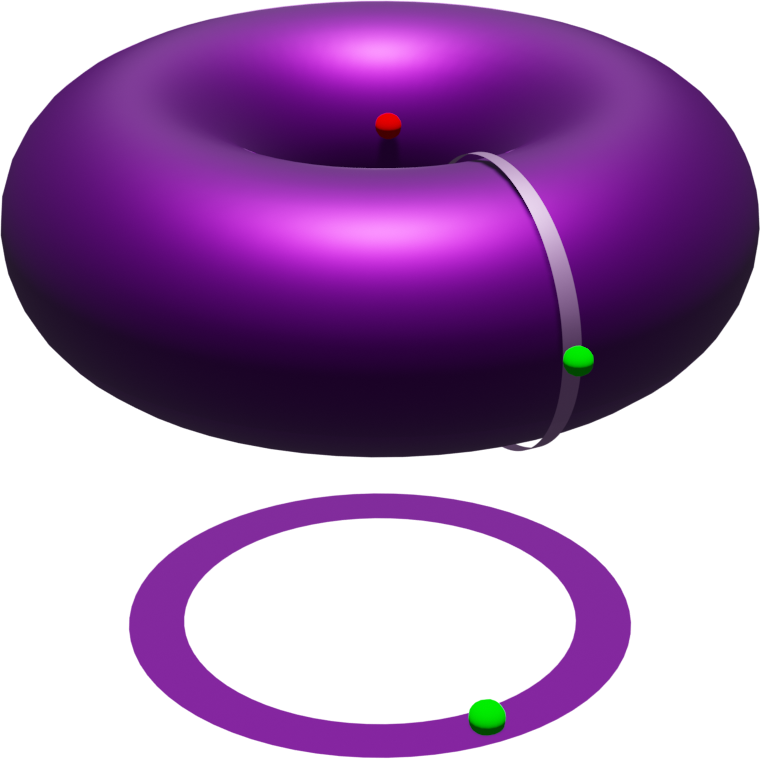}
        \caption{Fiber over base point.\label{fig:restriction:fiber}}
    \end{subfigure}\hfill
    \begin{subfigure}[t]{0.33\textwidth}
    \centering
        \vskip 0pt
        \includegraphics[width=\textwidth]{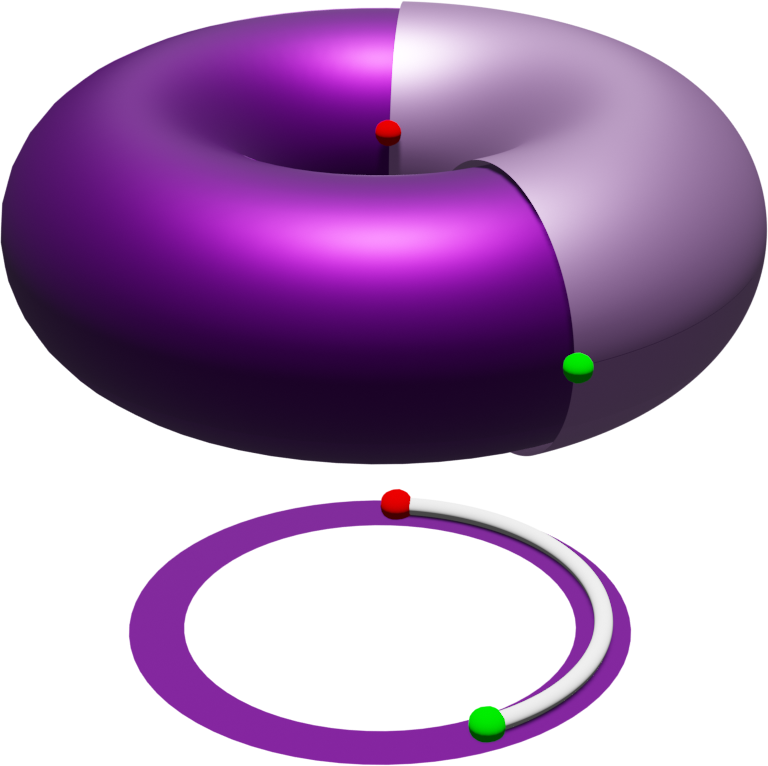}
        \caption{Path restriction over base path.\label{fig:restriction:path}}
    \end{subfigure}\hfill
    \begin{subfigure}[t]{0.33\textwidth}
    \centering
        \vskip 0pt
        \includegraphics[width=\textwidth]{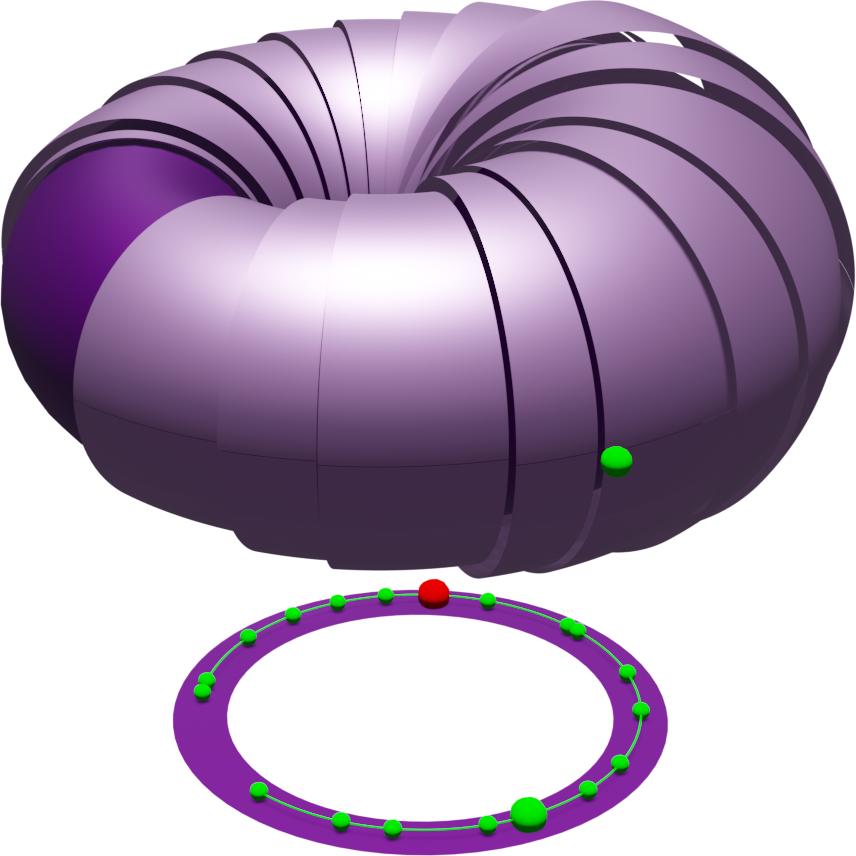}
        \caption{Graph restriction over sparse base graph.\label{fig:restriction:graph}}
    \end{subfigure}\hfill
    \caption{Fiber bundle restrictions on the fiber bundle $T^2 \rightarrow S^1$ with $T^2$ being the torus and $S^1$ being the circle. See text for clarification.\label{fig:restriction}}
\end{figure*}

%% file: src/method.tex
\section{Sparse Multilevel Roadmaps}

\input{src/pseudocode}

Let $(\xi, \xg, \X_1,\ldots,\X_K)$ be a fiber bundle sequence with $\xi$ and $\xg$ being start and goal state. Our task is to generalize the sparse roadmap planner \cite{dobson_2014} to fiber bundle sequences by growing $K$ graphs $(\G_1, \ldots, \G_K)$ on the bundle spaces $(\X_1,\ldots,\X_K)$, whereby we grow the $k$-th graph using restriction sampling \cite{Orthey2020IJRR} of the $(k-1)$-th graph. We call our algorithm the sparse multilevel roadmap planner (SMLR). SMLR depends on three parameters, the two parameters $\visRegion$ and $M$ from sparse roadmaps, and the additional parameter $\eta$, which we detail later.

We show the algorithm in Alg.~\ref{alg:smlr}. We start to create a priority queue (Line \algref{alg:smlr}{alg:smlr:priorityqueue}), which orders bundle spaces depending on an importance criterion $i$, which we detail later. We sort the queue such that the space with the maximum value is on top. We then iterate over the bundle spaces from $\X_1$ to $\X_K$ (Line \algref{alg:smlr}{alg:smlr:forcur}) and push the current space onto the priority queue with an importance of $1$ (Line \algref{alg:smlr}{alg:smlr:pushcur}). We then execute a section test (Line \algref{alg:smlr}{alg:smlr:section}), where we search for a feasible solution over the path restriction of the solution path (if any) on the previous bundle space $\X_{\text{cur}-1}$. The \textsc{SectionTest} method helps to overcome narrow passages, but is not essential for the understanding of this paper -- we use it as a black box within SMLR. Please see our previous publication \cite{Orthey2020TRO} for more information.

We then grow the roadmaps $(\G_1,\ldots,\G_{\text{cur}})$ as long as the planner terminate condition (PTC) of the current bundle space $\Xcur$ is not fulfilled (Line \algref{alg:smlr}{alg:smlr:while}). In our case, we terminate if a solution is found or if we reach either the infeasibility criterion or a time limit. Inside the while loop, we take the top bundle space $\Xtop$ with the largest importance value (Line \algref{alg:smlr}{alg:smlr:poptop}) and sample a random point using \textsc{RestrictionSampling} (Line \algref{alg:smlr}{alg:smlr:restrictionsampling}). We then add the point to the graph with \textsc{AddConditional} (Line \algref{alg:smlr}{alg:smlr:addconditional}), if it fulfills the sparseness condition \cite{dobson_2014}, which we detail in Sec.~\ref{sec:background}. Finally, we recompute the importance of the bundle space (Line
\algref{alg:smlr}{alg:smlr:importance}) and push the space back onto the queue (Line \algref{alg:smlr}{alg:smlr:pushtop}). 

The two methods \textsc{RestrictionSampling} and \textsc{ComputeImportance} are further detailed in the next two subsections. To facilitate understanding, we give first a brief overview of each. First, in \textsc{RestrictionSampling}, we restrict sampling on the bundle space by using information from the graph on its base space. We differ from dense roadmaps by using the visibility region of the sparse graph which depends on the visibility range $\visRegion$. Second, in \textsc{ComputeImportance}, we use the sampling density of the sparse graph together with the number of consecutive failures to estimate the importance of the bundle space and thereby its position in the priority queue. Next, we discuss each method in more detail and provide an analysis of the algorithm.

\subsection{Restriction Sampling with Visibility Regions}

Let $\X_k$ be a bundle space with graph $\G_k$, and let $\X_{k-1}$ be its base space with graph $\G_{k-1}$. To grow the graph $\G_k$, we use the framework of restriction sampling \cite{Orthey2020IJRR}. In restriction sampling, we sample states on $\X_k$ by uniformly sampling from the graph restriction of $\G_{k-1}$ (see Sec.~\ref{sec:background}). To give guarantees on asymptotic optimality, we would need the vertices of $\G_{k-1}$ to become dense in the free state space. 

To avoid using a dense graph for sampling \cite{dobson_2014} while giving guarantees on asymptotic near-optimality, we opt to exploit the graph visibility region. The visibility region of a graph $\G$ is the set $V(\G, \visRegion) = \{x \in \X \mid d(\x,e[I]) \leq \visRegion \text{ for some } e \text{ in } \G\}$, whereby $d$ is the metric on $\X$, $e$ is an edge from $\G$ and $e[I]$ is the image of the edge in $\X$.

\begin{wrapfigure}{r}{0.5\linewidth}
\includegraphics[width=0.95\linewidth]{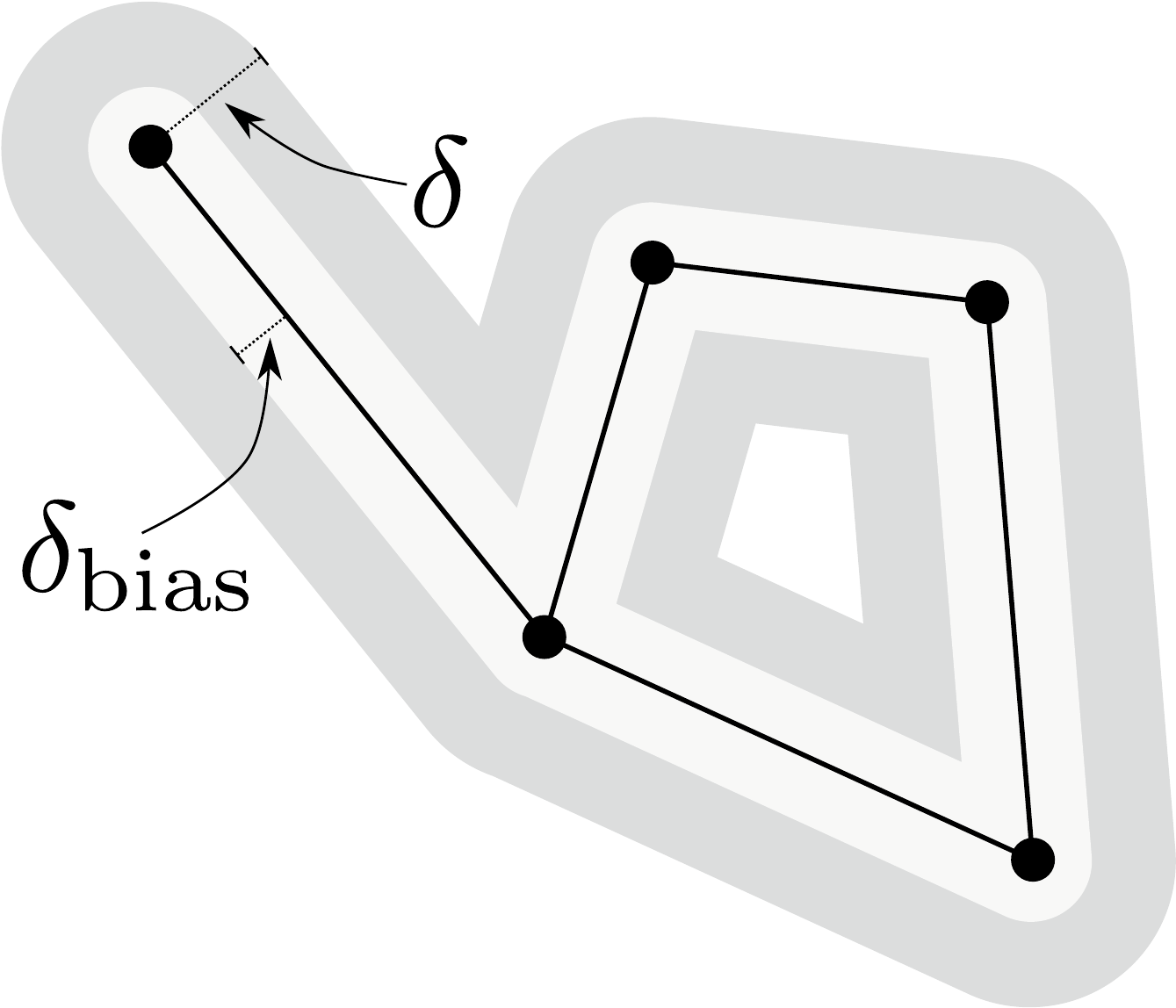}
\caption{Visibility region $V(\G, \visRegion)$ of a graph $\G$.\label{fig:visibilityregion}}
\end{wrapfigure}

To sample the graph visibility region, we use the  restriction sampling algorithm depicted in Alg. \ref{alg:restriction_sampling}. 
The algorithm requires an existing base graph $\G_{k-1}$ (Line \algref{alg:restriction_sampling}{alg:restriction_sampling:exists}), then samples a random state on a random edge (Line \algref{alg:restriction_sampling}{alg:restriction_sampling:sampleedge}).
Sampling the visibility region directly would be too uninformative. 
We thus use a smoothly varying parameter $\bias \in [0,\visRegion]$, which first restricts sampling to the sparse graph ($\bias=0$), then smoothly increase in each iteration until the whole visibility region $\delta$. 
This situation is visualized in Fig.~\ref{fig:visibilityregion}. To control the rate of change of $\bias$, we use the parameter $\eta$. 

In particular for narrow passages, it is often crucial to sample directly on the graph restriction. We thus sample the visibility region (Line \algref{alg:restriction_sampling}{alg:restriction_sampling:uniformnear}) only in a certain percentage of cases, depending on $\bias$. Once a base element is chosen, we sample a corresponding fiber space element (Line \algref{alg:restriction_sampling}{alg:restriction_sampling:samplefiber}), lift the states (Line \algref{alg:restriction_sampling}{alg:restriction_sampling:lift}) and return the state (Line \algref{alg:restriction_sampling}{alg:restriction_sampling:return}). If no base graph exists, we revert to a uniform sampling of the space (Line \algref{alg:restriction_sampling}{alg:restriction_sampling:nobase}). 

\subsection{Importance and Ordering of Bundle Spaces}

To grow sparse multilevel roadmaps, we need to decide which roadmap on which level we should grow next, i.e.~we need an ordering of bundle spaces. In prior work \cite{Orthey2020IJRR}, we advocated the use of an exponential importance criterion $i(\X_k) = 1/(|V_k|^{1/{n_k}}+1)$, with $|V_k|$ being the vertices on the graph $\G_k$ on $\X_k$ and $n_k$ being the dimensionality of $\X_k$, which was motivated by the sampling density of the graph which is proportional to $|V_k|^{1/n_k}$ \cite{Hastie2009}. 

However, sampling density is not good criterion for sparse roadmaps, because we care more about the coverage of the free space. To account for the coverage of the free space, we advocate an importance criterion using $M_k$, the number of consecutive sample failures. The number $M_k$ provides an estimate of the free space coverage, namely as the percentage $1-\frac{1}{M_k}$ \cite{Simeon2000}. The higher $M_k$, the less often we should sample $\X_k$. We formulate the importance criterion thus as 
\begin{equation}
    i(\X_k) = \frac{1}{M_k+1}.\label{eq:importance}
\end{equation}
Note that we stop the algorithm only if $M_k > M$ \emph{and} $\X_k$ is the current bundle space $\Xcur$. Since $i(\X_k)$ will eventually converge to zero, we ensure that every bundle space up until $k$ would be chosen infinitely many times. This is an important requirement to provide asymptotic guarantees of the algorithm.

\subsection{Analysis of Algorithm}

To prove SMLR to be asymptotically near-optimal and asymptotic sparse, we need to prove that restriction sampling with visibility regions is dense in the free state space of the last bundle space $\X$. Since the importance criterion in Eq.~\eqref{eq:importance} eventually converges to zero, we can thus ensure that we produce an infinite sampling sequence on the free state space $\Xf$. Therefore, when using sparse roadmap spanner \cite{dobson_2014} to grow the roadmap on $\X$, we retain all their properties, which include asymptotic near-optimality and asymptotic sparseness. However, we might reduce the number of vertices considerably.

Let us prove that restriction sampling with visibility regions is dense in the \emph{free} state space $\Xf$ on the fiber bundle $(\X,B,F,\pi)$. This argument can be applied recursively to prove the same for fiber bundle sequences \cite{Orthey2020IJRR}. Note that we use the set-theoretic definition of dense, which states that a set $A$ is dense in a space $\X$ if the intersection of $A$ with any non-empty open subset $U$ of $\X$ is non-empty \cite{munkres_1974}.

\begin{theorem}
Restriction sampling with visibility regions on $\X$ produces a sampling sequence $A = \{x_m\}$, which is dense in $\Xf$.
\end{theorem}

\begin{proof}
Let $U$ be an arbitrary open set in $\Xf$. Since $\pi$ is admissible, the projection $\pi(U)$ of $U$ onto $B$ is an open subset of the free base space \cite{Orthey2019}. Since uniform sampling on $B$ with visibility regions will eventually cover the free base space \cite{Simeon2000}, $\pi(U)$ will be a subset of the visibility region of the graph on $B$. When the number of samples goes to infinity, we revert to uniform sampling of the graph restriction and will thus sample $\pi(U)$ infinitely many times. By sampling the fiber over $\pi(U)$, we thus eventually obtain a sample $x$ in $U$. Since $U$ was arbitrary, the sequence is dense in $\Xf$.
\end{proof}

%% file: src/pseudocode.tex
\begin{algorithm}[t]
\caption{SMLR($\xi, \xg, \X_1,\ldots,\X_K$)}
\begin{algorithmic}[1]
  \State Let $\PriorityQueue$ be a \Call{priority\ queue}{}\Comment{Top is Max Value}\label{alg:smlr:priorityqueue}
  \For{$\Xcur$ in $\X_1,\ldots,\X_K$}\label{alg:smlr:forcur}
    \State $\PriorityQueue.\Call{push}{\Xcur, 1}$\label{alg:smlr:pushcur}
    \State $\Call{SectionTest}{\Xcur}$\label{alg:smlr:section}\Comment{See \cite{Orthey2020TRO}}
    \While{$\neg\Call{ptc}{\Xcur}$}\label{alg:smlr:while}
      \State $\Xtop = \PriorityQueue.\Call{pop}{}$\label{alg:smlr:poptop}
      \State $\xr \gets \Call{RestrictionSampling}{\Xtop}$\label{alg:smlr:restrictionsampling}
      \State $\Call{AddConditional}{\xr, \Gtop}$\label{alg:smlr:addconditional}     
      \State $i \gets \Call{ComputeImportance}{ \Xtop}$\label{alg:smlr:importance}\Comment{In $[0,1]$}
      \State $\PriorityQueue.\Call{push}{\Xtop, i}$\label{alg:smlr:pushtop}
    \EndWhile
  \EndFor
\end{algorithmic}\label{alg:smlr}
\end{algorithm}

\begin{algorithm}[t]
  \caption{RestrictionSampling($\Xk$)}
    \begin{algorithmic}[1]
    \If{\Call{Exists}{$\Xkk $}}\label{alg:restriction_sampling:exists}
        \def\xb{\x_{\text{base}}}
        \def\xf{\x_{\text{fiber}}}
        \State $e \gets \Call{SampleEdge}{\G_{k-1}}$\label{alg:restriction_sampling:sampleedge}
        \State $\xb \gets \Call{SampleUniform}{e}$\Comment{State on Edge}\label{alg:restriction_sampling:samplestate}
        \State $\bias \gets \Call{SmoothParameter}{0,\visRegion,\eta}$\label{alg:restriction_sampling:smoothparameter}
        \If{$\Call{Random}{0,1} < \bias/\visRegion$}
            \State $\xb \gets \Call{UniformNear}{\xb, \bias}$\label{alg:restriction_sampling:uniformnear}
        \EndIf
        \State $\xf \gets \Call{Sample}{\xb, \fiber_k}$\Comment{Element of $\fiber_k$}\label{alg:restriction_sampling:samplefiber}
        \State $\xr \gets \Call{Lift}{\xb, \xf}$\label{alg:restriction_sampling:lift}
        \Comment{Element of $\Xk$}
    \Else
        \State $\xr \gets \Call{Sample}{\Xk}$\label{alg:restriction_sampling:nobase}
    \EndIf
    \State \Return $\xr$\label{alg:restriction_sampling:return}
    \end{algorithmic}
  \label{alg:restriction_sampling}
\end{algorithm}

%% file: src/evaluations.tex
\section{Evaluation\label{sec:evaluation}}
\input{src/evaluations_table}

\input{images/evaluations/scenarios}

To evaluate \algorithmName, we compare its performance on eight scenarios against the algorithms SPARS and SPARS2 from the open motion planning library (OMPL). Both SPARS and SPARS2 are the only algorithms in OMPL we know of which can return on infeasible scenarios while not timing out. To ensure a fair comparison, we set the parameters of SMLR, SPARS and SPARS2 all to $M=1000$, $\visRegion = 0.25\mu$ with $\mu$ being the measure of the state space (removing effects stemming from different parameter values). For SMLR, we use the parameter $\eta = 1000$ which designates how fast we expand the graph visibility region for restriction sampling. 

While we like our algorithm to correctly declare an infeasible problem as infeasible, we also like to make sure that the algorithm does not show false negatives, i.e. declaring a feasible problem to be infeasible. To ensure correctness, we always use two similar scenarios, one which is feasible and one which is infeasible. For all scenarios, we run each algorithm $10$ times with a time limit of $60$s. Our setup is a 8GB RAM 4-core 2.5GHz laptop running Ubuntu 16.04.

\subsection{6-dimensional Bugtrap}

Our first scenario is the classical narrow-passage Bugtrap scenario, where a cylindrical robot (the bug with $6$ degrees of freedom (dof)) has to escape a spherical object with a narrow exit (the trap), as shown in Fig.~\ref{fig:scenarios}. We use two versions, a feasible one with a bug which barely fits through the exit, and an infeasible one where the bug does not fit. As a simplification, we use an inscribed sphere which we describe using the fiber bundle $SE(3) \rightarrow \R^3$. We show the results in Table~\ref{table:evaluation}. While SMLR can solve (on average) both scenarios in $4.37$ and $2.47$s, respectively, both SPARS and SPARS2 time out after $60$s. 

\subsection{6-dimensional Drone}

In the second scenario, we use a free-floating drone with $6$-dof. The drone has to traverse a room which is separated by a net. In the first version of the problem, we make the net large enough to let the drone fly trough (the feasible problem). In the second version, we make the net finely woven to prevent the drone from passing (the infeasible problem). As a simplification, we use a sphere at the center of the drone. We model this situation with the fiber bundle $SE(3) \rightarrow \R^3$. For the feasible scenario, all three planners solve the problem with SPARS2 taking $0.16$s, SMLR taking $0.23$s and SPARS taking $0.37$s. In the infeasible scenario, only SMLR solves the problem in $0.72$s, while SPARS and SPARS2 both time out. 
\subsection{7-dimensional KUKA LWR}

In the third scenario, we use a fixed-base KUKA LWR robot with $7$-dof, which has to transport a windshield through a gap in a wall (Fig.~\ref{fig:scenarios}). We create two versions, a feasible one with the gap in the wall and an infeasible one where we close the gap. As a simplification, we use a projection onto the first two links of the manipulator arm, which we describe using the fiber bundle $\R^7 \rightarrow \R^3$. With our algorithm SMLR, we can solve both scenarios in $1.42$s and $5.34$s. For the feasible scenario, SPARS requires $33.66$s (but times out in $4$ cases) and SPARS2 requires $34.86$s (but times out in $3$ cases). Both SPARS algorithms time out for the infeasible scenario in all runs. 

\subsection{34-dimensional PR2}

In the fourth scenario, we use the mobile-base PR2 robot with $34$-dof, which has to enter a room with a small opening as shown in Fig.~\ref{fig:scenarios}. We use again two scenarios, the feasible one with the opening and an infeasible one where we close the opening. As a simplification, we use two projections, first we remove the arms of the robot and second we project onto the mobile base. We model this situation by the fiber bundle sequence $\R^{34} \rightarrow \R^{7} \rightarrow \R^{2}$. Our algorithm SMLR requires $9.25$s to solve the feasible scenario (but times out in $1$ case) and it requires $0.32$s to terminate on the infeasible scenario. Both SPARS and SPARS2 cannot solve any of the runs in the time limit given.

%% file: src/evaluations_table.tex
\begin{table}[!t]
\centering
\renewcommand{\cellrotangle}{90}
\renewcommand\theadfont{\bfseries}
\settowidth{\rotheadsize}{\theadfont 06D Bugtrap [infeasible]}
\newcolumntype{Y}{>{\raggedleft\arraybackslash}X}
\footnotesize\centering
\renewcommand{\arraystretch}{2.5}
\setlength\tabcolsep{3pt}
\begin{tabulary}{\linewidth}{@{}LCCCCCCCC@{}}
\toprule
 & \rothead{06D Bugtrap [feasible]} & \rothead{06D Bugtrap [infeasible]} & \rothead{06D drone [feasible]} & \rothead{06D drone [infeasible]} & \rothead{07D kuka [feasible]} & \rothead{07D kuka [infeasible]} & \rothead{34D PR2 [feasible]} & \rothead{34D PR2 [infeasible]} \\ 
\midrule
 \mbox{SMLR (ours)} & \makecell{ \textbf{4.37} \\ {$\scriptscriptstyle 10\mid 0\mid 0$} } & \makecell{ \textbf{2.47} \\ {$\scriptscriptstyle 0\mid 10\mid 0$} } & \makecell{ 0.23 \\ {$\scriptscriptstyle 10\mid 0\mid 0$} } & \makecell{ \textbf{0.72} \\ {$\scriptscriptstyle 0\mid 10\mid 0$} } & \makecell{ \textbf{1.42} \\ {$\scriptscriptstyle 10\mid 0\mid 0$} } & \makecell{ \textbf{5.34} \\ {$\scriptscriptstyle 0\mid 10\mid 0$} } & \makecell{ \textbf{9.25} \\ {$\scriptscriptstyle 9\mid 0\mid 1$} } & \makecell{ \textbf{0.32} \\ {$\scriptscriptstyle 0\mid 10\mid 0$} } \\ 
 \mbox{SPARS} & \makecell{ 60.00 \\ {$\scriptscriptstyle 0\mid 0\mid 10$} } & \makecell{ 60.00 \\ {$\scriptscriptstyle 0\mid 0\mid 10$} } & \makecell{ 0.37 \\ {$\scriptscriptstyle 10\mid 0\mid 0$} } & \makecell{ 60.00 \\ {$\scriptscriptstyle 0\mid 0\mid 10$} } & \makecell{ 33.66 \\ {$\scriptscriptstyle 6\mid 0\mid 4$} } & \makecell{ 60.00 \\ {$\scriptscriptstyle 0\mid 0\mid 10$} } & \makecell{ 60.00 \\ {$\scriptscriptstyle 0\mid 0\mid 10$} } & \makecell{ 60.00 \\ {$\scriptscriptstyle 0\mid 0\mid 10$} } \\ 
 \mbox{SPARStwo} & \makecell{ 60.00 \\ {$\scriptscriptstyle 0\mid 0\mid 10$} } & \makecell{ 60.00 \\ {$\scriptscriptstyle 0\mid 0\mid 10$} } & \makecell{ \textbf{0.16} \\ {$\scriptscriptstyle 10\mid 0\mid 0$} } & \makecell{ 60.00 \\ {$\scriptscriptstyle 0\mid 0\mid 10$} } & \makecell{ 34.86 \\ {$\scriptscriptstyle 7\mid 0\mid 3$} } & \makecell{ 60.00 \\ {$\scriptscriptstyle 0\mid 0\mid 10$} } & \makecell{ 60.00 \\ {$\scriptscriptstyle 0\mid 0\mid 10$} } & \makecell{ 60.00 \\ {$\scriptscriptstyle 0\mid 0\mid 10$} } \\ 
\bottomrule
\end{tabulary}

\caption{Runtime (in seconds) of motion planner averaged over $10$ runs with $60$s time limit. We additionally show how often a planner terminated with a status of feasible$\vert$infeasible$\vert$timeout 
on each scenario.\label{table:evaluation}}
\end{table}

%% file: images/evaluations/scenarios.tex
\begin{figure*}
    \centering
\begin{subfigure}[t]{0.24\textwidth}
    \includegraphics[width=\textwidth]{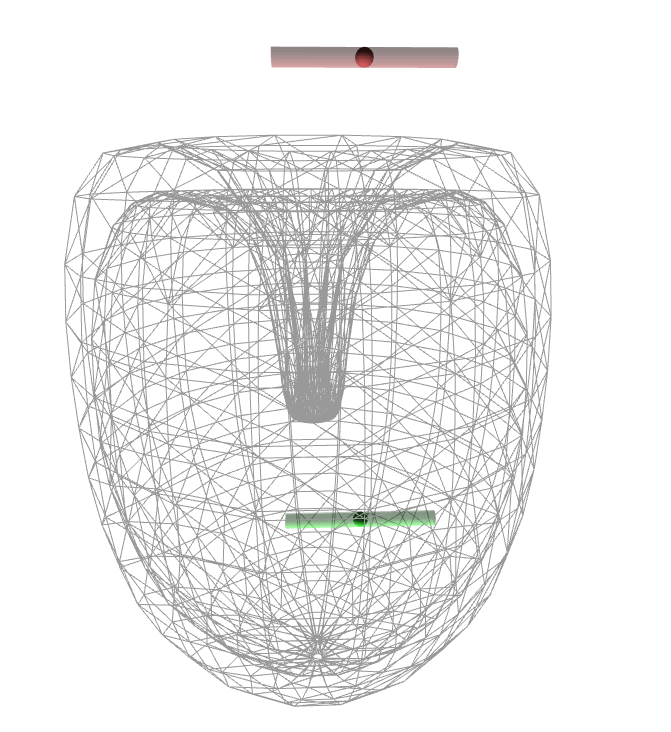}
    \includegraphics[width=\textwidth]{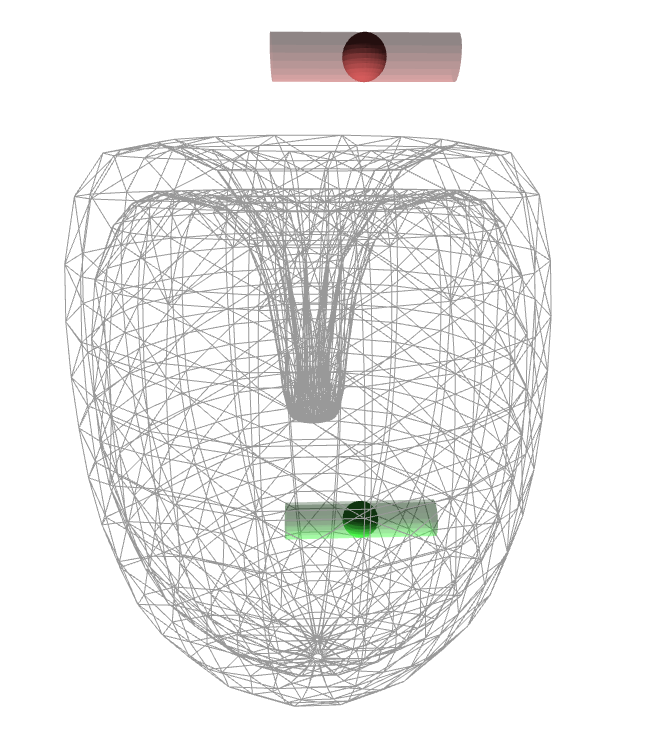}
\end{subfigure}
\begin{subfigure}[t]{0.24\textwidth}
    \includegraphics[width=\textwidth]{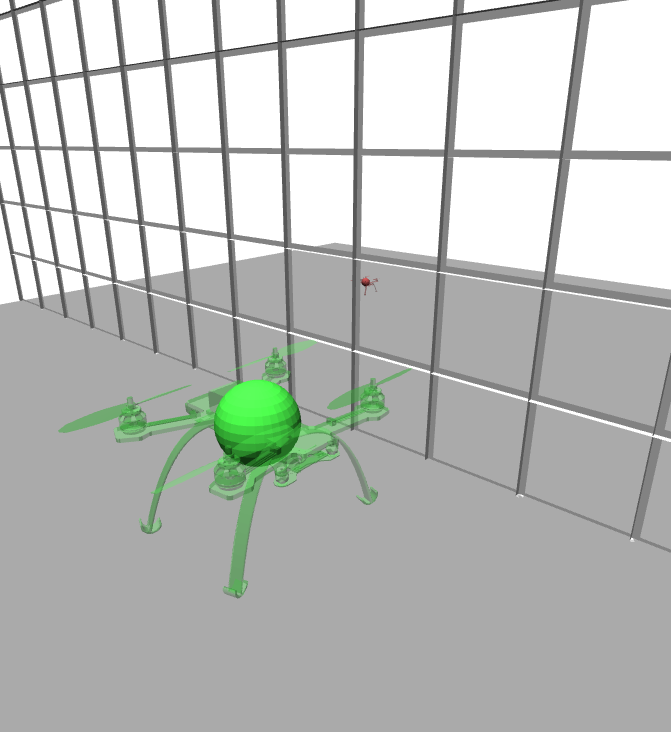}
    \includegraphics[width=\textwidth]{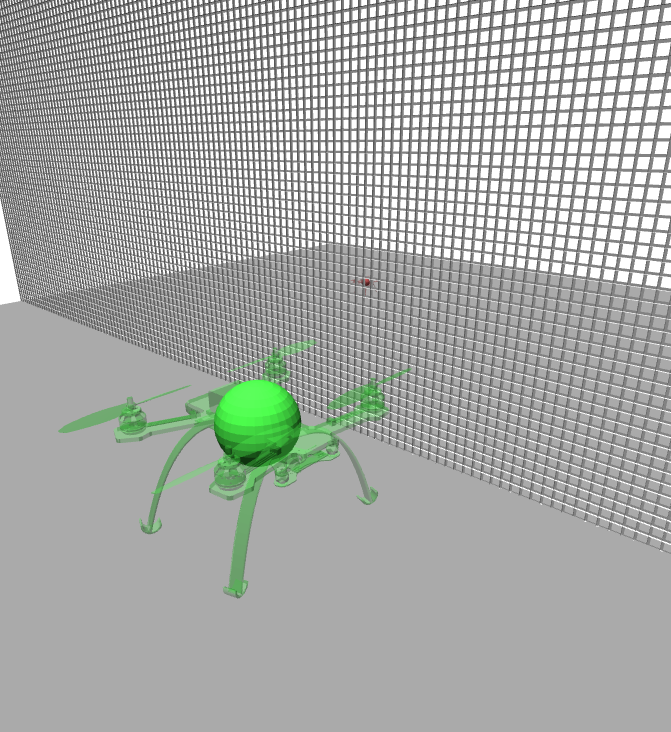}
\end{subfigure}
\begin{subfigure}[t]{0.24\textwidth}
    \includegraphics[width=\textwidth]{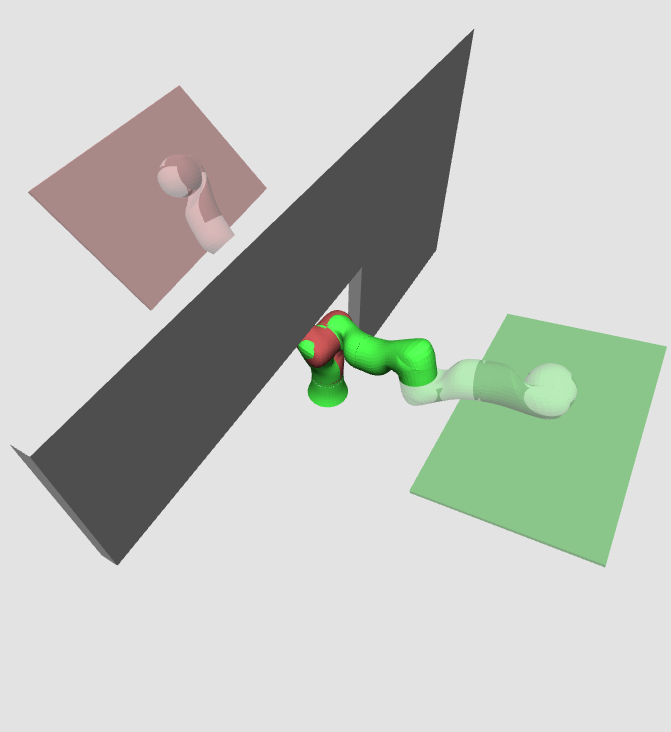}
    \includegraphics[width=\textwidth]{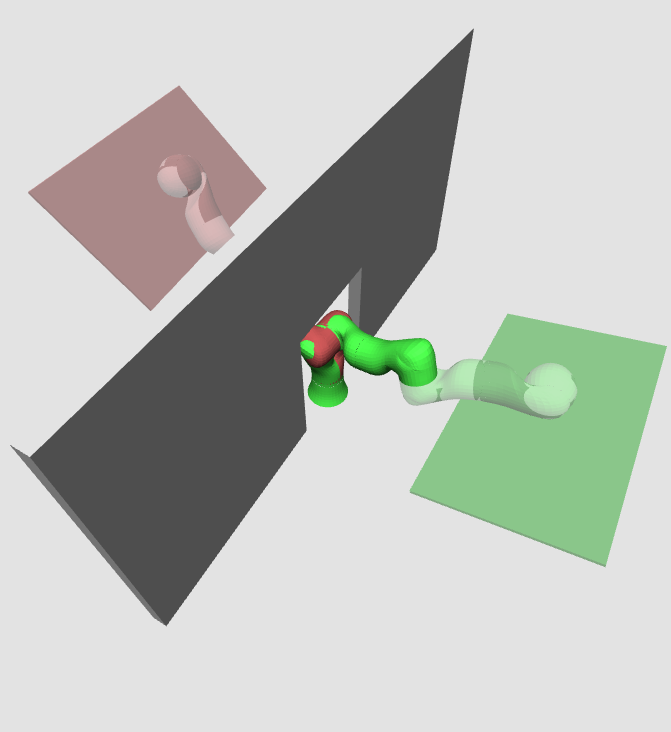}
\end{subfigure}
\begin{subfigure}[t]{0.24\textwidth}
    \includegraphics[width=\textwidth]{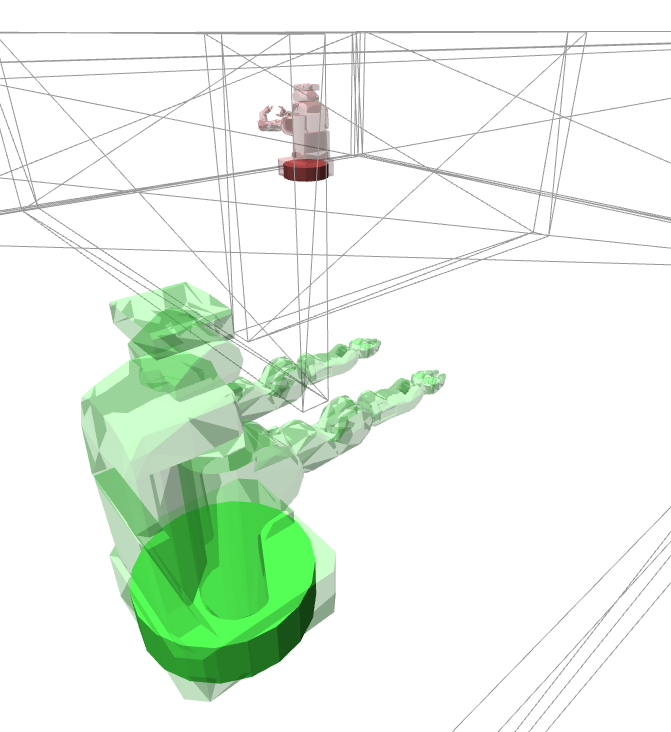}
    \includegraphics[width=\textwidth]{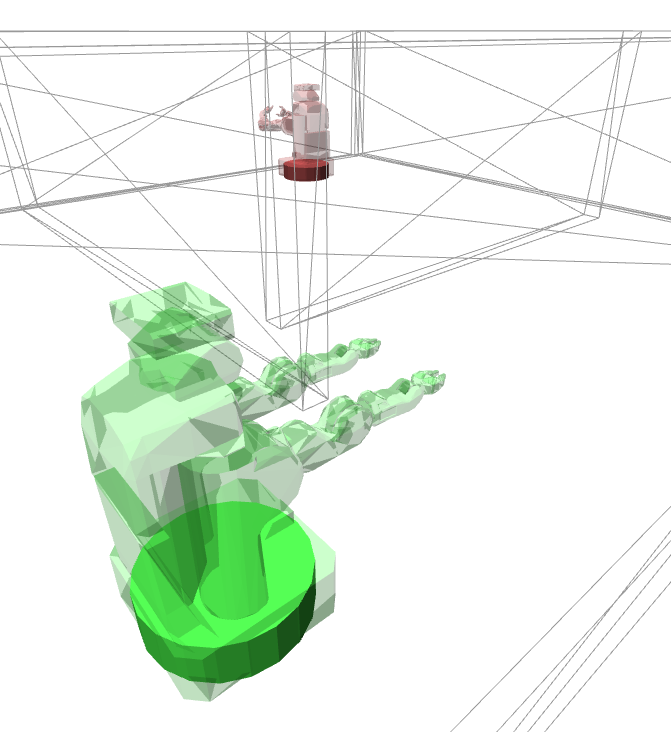}
\end{subfigure}
\caption{The eight scenarios used for evaluating our algorithm. Task for each scenario is to move robot from initial state (green) to goal state (red). Full robot geometry is shown as transparent color, the simplified version as non-transparent (Note that we use two simplifications for PR2, but only show one). \textbf{Top Row}: Feasible scenarios, where a solution exists. \textbf{Bottom Row}: Infeasible scenarios, where no solution exists. \textbf{Left to Right}: 6-dimensional Bugtrap, 6-dimensional drone, 7-dimensional KUKA LWR, 34-dimensional PR2.}
    \label{fig:scenarios}
\end{figure*}

%% file: src/conclusion.tex
\section{Conclusion}

We presented the sparse multilevel roadmap planner (\algorithmName), which we believe to be the first algorithm to generalize sparse roadmap spanners \cite{dobson_2014} to fiber bundles \cite{Orthey2020IJRR}, which are models of multilevel abstractions. Our algorithm exploits multilevel abstraction using the notion of restriction sampling with visibility regions. We have shown SMLR to be asymptotically near-optimal and asymptotically sparse by showing restriction sampling to produces a dense sampling sequence. In evaluations, we showed SMLR to efficiently and correctly terminate on feasible and infeasible problems, even when those problems have narrow passages, intricate geometries or state spaces with dimensions of up to $34$-dof. 



%% file: root.bbl
\begin{thebibliography}{10}
\providecommand{\url}[1]{#1}
\csname url@samestyle\endcsname
\providecommand{\newblock}{\relax}
\providecommand{\bibinfo}[2]{#2}
\providecommand{\BIBentrySTDinterwordspacing}{\spaceskip=0pt\relax}
\providecommand{\BIBentryALTinterwordstretchfactor}{4}
\providecommand{\BIBentryALTinterwordspacing}{\spaceskip=\fontdimen2\font plus
\BIBentryALTinterwordstretchfactor\fontdimen3\font minus
  \fontdimen4\font\relax}
\providecommand{\BIBforeignlanguage}[2]{{%
\expandafter\ifx\csname l@#1\endcsname\relax
\typeout{** WARNING: IEEEtranS.bst: No hyphenation pattern has been}%
\typeout{** loaded for the language `#1'. Using the pattern for}%
\typeout{** the default language instead.}%
\else
\language=\csname l@#1\endcsname
\fi
#2}}
\providecommand{\BIBdecl}{\relax}
\BIBdecl

\bibitem{Aine2016}
S.~Aine, S.~Swaminathan, V.~Narayanan, V.~Hwang, and M.~Likhachev,
  ``Multi-heuristic {A*},'' \emph{International Journal of Robotics Research},
  vol.~35, no. 1-3, pp. 224--243, 2016.

\bibitem{Brandao2020}
M.~Brandao and I.~Havoutis, ``Learning sequences of approximations for
  hierarchical motion planning,'' in \emph{International Conference on
  Automated Planning and Scheduling}, vol.~30, 2020, pp. 508--516.

\bibitem{dobson_2014}
A.~Dobson and K.~E. Bekris,
  ``\href{https://doi.org/10.1177/0278364913498292}{Sparse roadmap spanners for
  asymptotically near-optimal motion planning},'' \emph{International Journal
  of Robotics Research}, vol.~33, no.~1, pp. 18--47, 2014.

\bibitem{Du2020}
W.~Du, F.~Islam, and M.~Likhachev, ``Multi-resolution {A*},'' 2020,
  arXiv:2004.06684 [cs.RO].

\bibitem{Ferbach1997}
P.~Ferbach and J.~Barraquand, ``A method of progressive constraints for
  manipulation planning,'' \emph{{IEEE} Transactions on Robotics}, vol.~13,
  no.~4, pp. 473--485, 1997.

\bibitem{Hastie2009}
T.~Hastie, R.~Tibshirani, and J.~Friedman, \emph{The elements of statistical
  learning: data mining, inference, and prediction}.\hskip 1em plus 0.5em minus
  0.4em\relax Springer Science \& Business Media, 2009.

\bibitem{Ichnowski2019}
J.~Ichnowski and R.~Alterovitz, ``Multilevel incremental roadmap spanners for
  reactive motion planning,'' in \emph{IEEE International Conference on
  Intelligent Robots and Systems}.\hskip 1em plus 0.5em minus 0.4em\relax IEEE,
  2019, pp. 1504--1509.

\bibitem{Ichter2019}
B.~Ichter and M.~Pavone, ``Robot motion planning in learned latent spaces,''
  \emph{IEEE Robotics and Automation Letters}, vol.~4, no.~3, pp. 2407--2414,
  2019.

\bibitem{jaillet_2008}
L.~Jaillet and T.~Sim{\'e}on, ``Path deformation roadmaps: Compact graphs with
  useful cycles for motion planning,'' \emph{International Journal of Robotics
  Research}, 2008.

\bibitem{Janson2015}
L.~Janson, E.~Schmerling, A.~Clark, and M.~Pavone, ``Fast marching tree: A fast
  marching sampling-based method for optimal motion planning in many
  dimensions,'' \emph{International Journal of Robotics Research}, vol.~34,
  no.~7, pp. 883--921, 2015.

\bibitem{Kaelbling2011}
L.~P. Kaelbling and T.~Lozano-P{\'e}rez, ``Hierarchical task and motion
  planning in the now,'' in \emph{IEEE International Conference on Robotics and
  Automation}.\hskip 1em plus 0.5em minus 0.4em\relax IEEE, 2011, pp.
  1470--1477.

\bibitem{Karaman2011}
S.~Karaman and E.~Frazzoli, ``Sampling-based algorithms for optimal motion
  planning,'' \emph{International Journal of Robotics Research}, vol.~30,
  no.~7, pp. 846--894, 2011.

\bibitem{Kavraki1996}
L.~E. Kavraki, P.~Svestka, J.-C. Latombe, and M.~H. Overmars,
  ``\href{https://ieeexplore.ieee.org/document/508439/}{Probabilistic roadmaps
  for path planning in high-dimensional configuration spaces},'' \emph{{IEEE}
  Transactions on Robotics}, vol.~12, no.~4, pp. 566--580, 1996.

\bibitem{Kuffner2000}
J.~J. Kuffner and S.~M. LaValle,
  ``\href{https://ieeexplore.ieee.org/document/844730/}{RRT-connect: An
  efficient approach to single-query path planning},'' in \emph{IEEE
  International Conference on Robotics and Automation}, vol.~2, 2000, pp.
  995--1001.

\bibitem{lavalle_2006}
S.~M. LaValle, \emph{\href{http://planning.cs.uiuc.edu/}{Planning
  Algorithms}}.\hskip 1em plus 0.5em minus 0.4em\relax Cambridge University
  Press, 2006.

\bibitem{lee_2003}
J.~M. Lee,
  \emph{\href{https://www.springer.com/jp/book/9780387217529}{Introduction to
  Smooth Manifolds}}.\hskip 1em plus 0.5em minus 0.4em\relax New York, NY:
  Springer New York, 2003.

\bibitem{Marble2013}
J.~D. Marble and K.~E. Bekris, ``Asymptotically near-optimal planning with
  probabilistic roadmap spanners,'' \emph{{IEEE} Transactions on Robotics},
  vol.~29, no.~2, pp. 432--444, 2013.

\bibitem{munkres_1974}
J.~R. Munkres, \emph{Topology: a first course}.\hskip 1em plus 0.5em minus
  0.4em\relax Prentice-Hall, 1974.

\bibitem{Murray2020}
S.~Murray, G.~D. Konidaris, and D.~J. Sorin, ``Roadmap subsampling for changing
  environments,'' in \emph{IEEE International Conference on Intelligent Robots
  and Systems}.\hskip 1em plus 0.5em minus 0.4em\relax IEEE, 2020.

\bibitem{Nieuwenhuisen2004}
D.~Nieuwenhuisen and M.~H. Overmars, ``Useful cycles in probabilistic roadmap
  graphs,'' in \emph{IEEE International Conference on Robotics and Automation},
  vol.~1.\hskip 1em plus 0.5em minus 0.4em\relax IEEE, 2004, pp. 446--452.

\bibitem{Orourke1987}
J.~O'{R}ourke, \emph{Art gallery theorems and algorithms}.\hskip 1em plus 0.5em
  minus 0.4em\relax Oxford University Press Oxford, 1987, vol.~57.

\bibitem{Orthey2020IJRR}
A.~Orthey, S.~Akbar, and M.~Toussaint, ``Multilevel motion planning: A fiber
  bundle formulation,'' 2020, arXiv:2007.09435 [cs.RO].

\bibitem{Orthey2018}
A.~Orthey, A.~Escande, and E.~Yoshida, ``Quotient-space motion planning,'' in
  \emph{IEEE International Conference on Intelligent Robots and Systems}.\hskip
  1em plus 0.5em minus 0.4em\relax IEEE, 2018, pp. 8089--8096.

\bibitem{Orthey2019}
A.~Orthey and M.~Toussaint, ``Rapidly-exploring quotient-space trees: Motion
  planning using sequential simplifications,'' \emph{International Symposium of
  Robotics Research}, 2019.

\bibitem{Orthey2020TRO}
------, ``Section patterns: Efficiently solving narrow passage problems using
  multilevel motion planning,'' 2020, arXiv:2010.14524 [cs.RO].

\bibitem{Reid2019}
W.~Reid, R.~Fitch, A.~H. G{\"o}kto{\u{g}}an, and S.~Sukkarieh, ``Sampling-based
  hierarchical motion planning for a reconfigurable wheel-on-leg planetary
  analogue exploration rover,'' \emph{Journal of Field Robotics}, 2019.

\bibitem{Reid2020}
W.~Reid, R.~Fitch, A.~H. G{\"o}ktoǧgan, and S.~Sukkarieh, ``Motion planning
  for reconfigurable mobile robots using hierarchical fast marching trees,'' in
  \emph{Algorithmic Foundations of Robotics XII}.\hskip 1em plus 0.5em minus
  0.4em\relax Springer, 2020, pp. 656--671.

\bibitem{Rickert2014}
M.~Rickert, A.~Sieverling, and O.~Brock,
  ``\href{https://ieeexplore.ieee.org/document/6871370/}{Balancing exploration
  and exploitation in sampling-based motion planning},'' \emph{{IEEE}
  Transactions on Robotics}, vol.~30, no.~6, pp. 1305--1317, 2014.

\bibitem{Salzman2014}
O.~Salzman, D.~Shaharabani, P.~K. Agarwal, and D.~Halperin, ``Sparsification of
  motion-planning roadmaps by edge contraction,'' \emph{International Journal
  of Robotics Research}, vol.~33, no.~14, pp. 1711--1725, 2014.

\bibitem{Saund2020}
B.~Saund and D.~Berenson, ``Fast planning over roadmaps via selective
  densification,'' \emph{IEEE Robotics and Automation Letters}, vol.~5, no.~2,
  pp. 2873--2880, 2020.

\bibitem{Schmitzberger2002}
E.~Schmitzberger, J.-L. Bouchet, M.~Dufaut, D.~Wolf, and R.~Husson, ``Capture
  of homotopy classes with probabilistic road map,'' in \emph{IEEE
  International Conference on Intelligent Robots and Systems}, vol.~3.\hskip
  1em plus 0.5em minus 0.4em\relax IEEE, 2002, pp. 2317--2322.

\bibitem{Sekhavat1998}
S.~Sekhavat, P.~Svestka, J.-P. Laumond, and M.~H. Overmars, ``Multilevel path
  planning for nonholonomic robots using semiholonomic subsystems,''
  \emph{International Journal of Robotics Research}, vol.~17, no.~8, pp.
  840--857, 1998.

\bibitem{Simeon2000}
T.~Sim{\'e}on, J.-P. Laumond, and C.~Nissoux, ``Visibility-based probabilistic
  roadmaps for motion planning,'' \emph{Advanced Robotics}, vol.~14, no.~6, pp.
  477--493, 2000.

\bibitem{simeon_2002}
T.~Sim{\'e}on, S.~Leroy, and J.~P. Laumond, ``Path coordination for multiple
  mobile robots: A resolution-complete algorithm,'' \emph{IEEE Transactions on
  Robotics and Automation}, vol.~18, no.~1, pp. 42--49, 2002.

\bibitem{steenrod_1951}
N.~E. Steenrod, ``The topology of fibre bundles,'' 1951.

\bibitem{Tonneau2018}
S.~Tonneau, A.~D. Prete, J.~Pettré, C.~Park, D.~Manocha, and N.~Mansard,
  ``\href{https://ieeexplore.ieee.org/document/8341955/}{An Efficient Acyclic
  Contact Planner for Multiped Robots},'' \emph{{IEEE} Transactions on
  Robotics}, vol.~34, no.~3, pp. 586--601, June 2018.

\bibitem{Toussaint2018}
M.~Toussaint, K.~R. Allen, K.~A. Smith, and J.~B. Tenenbaum, ``Differentiable
  physics and stable modes for tool-use and manipulation planning,'' in
  \emph{Robotics: Science and Systems}, 2018.

\bibitem{Vidal2019}
E.~Vidal, M.~Moll, N.~Palomeras, J.~D. Hern{\'a}ndez, M.~Carreras, and L.~E.
  Kavraki, ``Online multilayered motion planning with dynamic constraints for
  autonomous underwater vehicles,'' in \emph{IEEE International Conference on
  Robotics and Automation}.\hskip 1em plus 0.5em minus 0.4em\relax IEEE, 2019,
  pp. 8936--8942.

\bibitem{Wang2015}
W.~Wang, D.~Balkcom, and A.~Chakrabarti, ``A fast online spanner for roadmap
  construction,'' \emph{International Journal of Robotics Research}, vol.~34,
  no.~11, pp. 1418--1432, 2015.

\end{thebibliography}
